\documentclass[lettersize,journal]{IEEEtran}
\usepackage{amsmath,amsfonts}
\usepackage{algorithm}
\usepackage{array}
\usepackage[caption=false,font=normalsize,labelfont=sf,textfont=sf]{subfig}
\usepackage{textcomp}

\usepackage{stfloats}
\usepackage{url}
\usepackage{verbatim}
\usepackage{booktabs}
\usepackage{multirow}
\usepackage{multicol}
\usepackage{tikz}
\usepackage{comment}

\usetikzlibrary{shapes.geometric}

\usepackage{graphicx}
\usepackage{cite}
\def\BibTeX{{\rm B\kern-.05em{\sc i\kern-.025em b}\kern-.08em
    T\kern-.1667em\lower.7ex\hbox{E}\kern-.125emX}}
\usepackage{balance}
\usepackage{algpseudocode}
\newcommand{\indep}{\rotatebox[origin=c]{90}{$\models$}}

\usepackage{caption}
\usepackage{amsthm}
\newtheorem{definition}{Definition}
\newtheorem{remark}{Remark}  
\newtheorem{example}{Example}
\newtheorem{proposition}{Proposition}
\newtheorem{theorem}{Theorem}
\newtheorem{corollary}{Corollary}

\usepackage{amssymb}

\newtheorem{lemma}{Lemma}

\begin{document}
\title{Decomposition for Bayesian Networks: Local and Parallel Inference}
\author{Pei Heng, Xinyi Hu, and Yi Sun 
\thanks{\noindent\scriptsize$\bullet$ Pei Heng is with the School of Mathematics and Statistics and KLAS, Northeast Normal University, Changchun, China. (Email: peiheng@nenu.edu.cn)

\scriptsize$\bullet$  Xinyi Hu is with the College of Mathematics and System Sciences, Xinjiang University, Urumqi, China. (Email: 107552303491@stu.xju.edu.cn)

\scriptsize$\bullet$  Yi Sun is with the Institute of Statistics and Data Science, Xinjiang University of Finance and Economics, Urumqi, China. (Email: brian@xjufe.edu.cn)

 
\scriptsize$\bullet$ Yi Sun is the corresponding author.
}}

\markboth{IEEE Transactions on Pattern Analysis and Machine Intelligence, ~Vol.~XX, No.~XX, 2026}%
{Heng \MakeLowercase{\textit{et al.}}: Decomposition for Bayesian Networks: Local and Parallel Inference}

\maketitle

\begin{abstract}

Probabilistic inference in high-dimensional Bayesian networks is difficult because exact manipulation of the joint distribution scales exponentially with network size. We propose a decomposition framework based on directed convex subgraphs and introduce a minimal d-decomposition tree. Together, they provide a principled alternative to classical junction-tree constructions. The proposed framework represents the joint distribution by lower-dimensional sub-models that can be learned and stored separately. This decomposition reduces computational cost and naturally enables parallel computation. Based on a minimal d-decomposition tree, we further develop two parallel algorithms for parameter estimation and probabilistic inference. Experiments show that the proposed method substantially improves computational efficiency over junction-tree methods while maintaining inference accuracy, especially for low-dimensional queries.
 
\end{abstract}


\section{Introduction}

Bayesian networks (BNs) are probabilistic graphical models that represent conditional dependence structures among random variables by directed acyclic graphs (DAGs) \cite{2009Probabilistic}. This modelling framework provides a principled and interpretable theoretical basis for representing complex dependencies in high-dimensional data \cite{pearl2014probabilistic}. It also offers a mathematically transparent framework for probabilistic representation and inference under uncertainty. Consequently, they have been widely applied in data analysis and related engineering domains, including machine learning \cite{maxwell1997efficient,friedman2003being}, biomedical engineering \cite{pang2004computerized,roberts2006bayesian}, decision support systems \cite{janssens2006integrating,weber2008component}, and reliability analysis \cite{mi2012reliability,langseth2007applications}.

Despite these advantages, learning and inference remain challenging in high-dimensional BNs\cite{2009Probabilistic}. High-dimensional settings are often characterised by data sparsity, complex non-linear dependencies among variables, and substantial computational costs. These factors make direct estimation of the joint distribution difficult and call for more tractable methods.

In this work, we address the problem of decomposing a Bayesian network into low-dimensional sub-networks to facilitate efficient inference. Formally, let $\mathcal{M}(\mathcal{G}) = (\mathcal{G}, \mathcal{P}(\mathcal{G}))$ be a Bayesian network, where $\mathcal{G}=(V,E)$ is the underlying DAG and each $P(X_V)\in \mathcal{M}(\mathcal{G})$ satisfies the Markov property encoded by
$\mathcal{G}$. By introducing a directed decomposer $S$, the network can be partitioned into two sub-networks, $\mathcal{M}(\mathcal{G}_{A\cup S})$ and $\mathcal{M}(\mathcal{G}_{B\cup S})$, with $A, B, S \subseteq V$ being pairwise disjoint. This decomposition can be applied recursively until no further splitting is possible, yielding a collection of low-dimensional sub-networks. Computing and storing statistics on these sub-networks then allows subsequent inference and related tasks to be carried out more efficiently. This decomposition principle is analogous to atom decomposition \cite{berry2010introduction} in undirected graphical models, where the decomposer $S$ corresponds to a clique minimal separator. Such atom decompositions have been widely used for sampling \cite{carvalho2007simulation,wang2010simulation} and iterative proportional fitting \cite{xu2011improved}. These applications illustrate the computational value of structured subnetwork decomposition.

{\bf Related work}. Because decomposition is central to scalable learning and inference in high-dimensional BNs, many studies have developed decomposition-based methods. An early line of work is Pearl’s \cite{pearl2014probabilistic} belief propagation algorithm [15], which builds a junction tree and calibrates it by message passing. This structure provides an effective mechanism for caching intermediate computation results. Subsequently, Wu \cite{wu2007lossless} proposed a lossless decomposition method for BNs based on a divide-and-conquer strategy. This approach first performs global parameter learning and then factorizes the potentials on junction tree separators, distributing them to adjacent tree nodes. It was further shown that the conditional independences encoded in the original BN are fully preserved in the decomposed sub-networks. Despite their effectiveness, both approaches still follow a global paradigm. 
After decomposition, the subnetwork distributions still depend on globally learned quantities for initialization. As a result, these approaches do not achieve strictly local modelling or fully independent inference.

Other studies instead focus on the structure and statistical properties of the sub-models themselves. Kim and Kim \cite{kim2006divide} proposed the splitter decomposition method, which adds extra edges within the decomposed sub-models to ensure consistency of marginal distributions, allowing the sub-models to be parametrized directly from local data. However, this strategy inevitably increases the structural complexity of the sub-models, leading to higher data storage and computational costs. The underlying reason for such “edge-adding” operations is that DAGs are not closed under marginalization \cite{richardson2002ancestral,evans2016graphs}, which highlights the inherent difficulty of decomposition in BNs.

To address these limitations, Li and Guo \cite{li2013decomposition} proposed decomposing a BN using complete d-separators and showed that the resulting sub-models naturally satisfy conditional independence constraints. Their work identified a structural mechanism in BNs analogous to atom decomposition \cite{berry2010introduction} in undirected graphical models, allowing sub-models to perform independent parameter learning and storage based solely on local data, while remaining usable in subsequent inference. However, their method requires the d-decomposers to satisfy a completeness condition, which is not necessary for decomposability, thereby limiting its practical applicability.

{\bf Collapsibility}. We next examine the decomposed sub-models from the perspective of collapsibility. Collapsibility means that the inference results obtained from a sub-model are identical to the marginalization results of the original model over the same set of variables. This property is a necessary condition for the validity of decomposition. Xie and Geng \cite{xie2009collapsibility} provided a formal definition of collapsibility for BNs. Given a BN $\mathcal{M}(\mathcal{G})$ and a variable set $A$, the model $\mathcal{M}(\mathcal{G})$ is said to be collapsible onto $A$ if
\[
\mathcal{M}(\mathcal{G}, A) = \mathcal{M}(\mathcal{G}_A),
\]
where $\mathcal{M}(\mathcal{G}, A)$ denotes the marginal model of $\mathcal{M}(\mathcal{G})$ over $A$, and $\mathcal{M}(\mathcal{G}_A)$ denotes the sub-model constructed from the induced subgraph $\mathcal{G}_A$.

If the decomposer is chosen so that the resulting sub-models are collapsible, parameter estimation can be carried out on smaller sub-models without losing consistency with the full model. More recently, Heng et al. \cite{heng2026structural} further characterised collapsibility from a graph-theoretic perspective. They showed that a sub-model $\mathcal{M}(\mathcal{G}_A)$ is collapsible if the subgraph is a \emph{directed convex subgraph} in $\mathcal{G}$; for details of directed convex subgraphs, see Section~\ref{sec-3-con}.

{\bf Main contributions}. Our first contribution is a graph-theoretic characterization of d-decomposers: a subset is a d-decomposer if and only if it is a directed convex d-separator. Based on this observation, we propose a decomposition framework for BNs, in which the resulting subgraphs naturally satisfy collapsibility. Theorem~\ref{thm-1} gives the induced factorization of the joint distribution and clarifies the statistical role of the resulting sub-models.

We further develop an efficient decomposition algorithm and organize the resulting subgraphs into a minimal d-decomposition tree, where each d-decomposer corresponds to a d-convex minimal d-separator. Building on this structure, we develop two algorithms that leverage a minimal d-decomposition tree for parallel parameter estimation and probabilistic inference in high-dimensional networks. In the experimental section, we conduct large-scale simulations on discrete and Gaussian Bayesian networks. The results show that our method is substantially more efficient than existing junction-tree-based approaches in both parameter estimation and inference, especially for low-dimensional queries.

The remainder of this paper is organised as follows. Section~\ref{sec-2} introduces the necessary notation and background. Section \ref{sec-3} introduces directed decomposition for Bayesian networks and establishes its basic properties. It also constructs a minimal d-decomposition tree and presents a pruning rule for inference. Section~\ref{sec-4} describes two algorithms for parallel parameter learning and probabilistic inference using a minimal d-decomposition tree. Section~\ref{sec-5} evaluates the performance of a minimal d-decomposition tree-based method against standard approaches in parameter estimation and inference through empirical experiments. Finally, Section~\ref{sec-6} concludes the paper and provides a brief discussion.

\section{Preliminaries}\label{sec-2}

We begin with the notation and definitions used throughout the paper.

\subsection{Basic Terminology}

A directed acyclic graph (DAG) is denoted by $\mathcal{G} = (V, E)$, where $V = \{v_1,v_2, \ldots, v_n\}$ is the vertex set and $E$ is the directed edge set. A directed edge $(u, v) \in E$ indicates that $u$ is a parent of $v$ and that $v$ is a child of $u$. The sets of parents and children of a vertex $u$ in $\mathcal{G}$ are denoted by $pa_{\mathcal{G}}(u)$ and $ch_{\mathcal{G}}(u)$, respectively. The family of a vertex $v$, denoted by $fa_{\mathcal{G}}(v)$, is defined as the union of $pa_{\mathcal{G}}(v)$ and $\{v\}$. Two vertices $u$ and $v$ are said to be adjacent, written $u \sim v$, if there is an edge between them. A subgraph $\mathcal{G}_A = (A, E_A)$ of $\mathcal{G}$ is defined as the induced subgraph with vertex set $A \subseteq V$ and edge set $E_A = E \cap (A \times A)$.

A path from $u$ to $v$, denoted by $l_{uv}$, is a sequence of  $(u = v_0, e_0, v_1, e_1, \ldots, v_{k-1}, e_{k-1}, v_k = v)$ in which each edge $e_j$ connects $v_j$ and $v_{j+1}$ for $j = 0, 1, \ldots, k-1$. If all edges are oriented as $(v_j, v_{j+1}) \in E$, then $l_{uv}$ is called a directed path. In this case, $u$ is called an ancestor of $v$, and the set of all ancestors of $v$ in $\mathcal{G}$ is denoted by $an_{\mathcal{G}}(v)$. 

The set of all internal vertices on the path $l_{uv}$, excluding the endpoints $u$ and $v$, is denoted by $V^o(l_{uv}) = V(l_{uv}) \setminus \{u, v\}$. 
An internal vertex $w$ on $l_{uv}$ is a collider if the two edges incident to $w$ both point toward $w$, i.e.,  $\{x \rightarrow w \leftarrow y\} \subseteq l_{uv}$. Otherwise, $w$ is a non-collider on $l_{uv}$. The sets of all collider and non-collider vertices on the path $l_{uv}$ are denoted by $V^c(l_{uv})$ and $V^{nc}(l_{uv})$, respectively.

 For a subset $A \subseteq V$, we define:

\begin{itemize}
	\item \textbf{Parent set}: $ pa_\mathcal{G}(A) = \bigcup_{v \in A} pa_\mathcal{G}(v)  \backslash A $.  
	\item \textbf{Child set}: $ ch_\mathcal{G}(A) = \bigcup_{v \in A} ch_\mathcal{G}(v) \backslash A $.  
	\item \textbf{Ancestral set}: $ An_\mathcal{G}(A) = \left( \bigcup_{v \in A} an_\mathcal{G}(v) \right) \cup A $.  
	\item \textbf{Markov boundary}: 
	\[
	mb_\mathcal{G}(A) = \left( 
	pa_\mathcal{G}(A) \cup ch_\mathcal{G}(A) \cup \bigcup_{w \in ch_\mathcal{G}(A)} pa_{\mathcal{G}}(w) 
	\right) \setminus A.
	\]
	
\end{itemize}

\begin{table*}[t]
\centering
\caption{Notation Reference Sheet }
\begin{tabular}{lp{0.7\textwidth}}
\hline
\textbf{Notation} & \textbf{Context-Specific Meaning} \\
\hline
$pa_\mathcal{G}(v), ch_\mathcal{G}(v), fa_\mathcal{G}(v), an_\mathcal{G}(v)$ & Parent, child, family ($fa_\mathcal{G}(v) = pa_\mathcal{G}(v) \cup \{v\}$), and ancestor sets of vertex $v$ in DAG $\mathcal{G}$ \\
$u \sim v$ & Vertices $u$ and $v$ are adjacent \\
$l_{uv}$ & Path connecting vertices $u$ and $v$ \\
$V^o(l_{uv}), V^c(l_{uv}), V^{nc}(l_{uv})$ & Internal vertices on path $l_{uv}$, with $V^c(l_{uv})$ being colliders and $V^{nc}(l_{uv})$ non-colliders \\
$\mathcal{G}_A = (A, E_A)$ & Induced subgraph on vertex subset $A \subseteq V$ \\
$pa_\mathcal{G}(A), ch_\mathcal{G}(A), An_\mathcal{G}(A), mb_\mathcal{G}(A)$ & Parent set, child set, ancestral set, and Markov boundary of vertex subset $A \subseteq V$ \\
$X \indep Y \mid Z [\mathcal{G}]$ & $X$ and $Y$ are d-separated by $Z$ in DAG $\mathcal{G}$ \\
$\mathcal{G}^m$ & Moral graph of $\mathcal{G}$ \\
$I(\mathcal{G}), I(\mathcal{G})_A$ & Independence model induced by DAG $\mathcal{G}$ and its marginal on $A \subseteq V$ \\
$\mathcal{M}(\mathcal{G}), \mathcal{M}(\mathcal{G}, A)$ & Bayesian network and its marginal over $A \subseteq V$ \\
\hline
\end{tabular}
\label{tab-notation}
\end{table*}

 \subsection{Independence Model and Marginal Independence Model}

For pairwise disjoint subsets $X, Y, Z \subseteq V$, we say that $X$ is d-separated from $Y$ given $Z$, denoted by $X \indep Y \mid Z \ [\mathcal{G}]$, if for every path $l_{xy}$ between any $x \in X$ and $y \in Y$, at least one of the following conditions holds:
\begin{itemize}
	\item $V^{nc}(l_{xy}) \cap Z \neq \emptyset$,  
	\item $V^{c}(l_{xy}) \cap An_{\mathcal{G}}(Z) = \emptyset$.  
\end{itemize}

In this case, $Z$ is a d-separator of $X$ and $Y$. It is a minimal d-separator if no proper subset $Z' \subsetneq Z$ d-separates $X$ and $Y$. The moral graph of $\mathcal{G}$, denoted by $\mathcal{G}^m$, is the undirected graph formed by first adding an undirected edge between each pair of non-adjacent parents that share a common child, and then replacing every directed edge with an undirected edge.

Using d-separation, we define the \emph{independence model} induced by $\mathcal{G}$ as
\[
\begin{split}
I(\mathcal{G}) = \Big\{ (X,Y,Z) \,\Big|\, & X \indep Y \mid Z \, [\mathcal{G}], \\
& X, Y, Z \subseteq V \text{ are pairwise disjoint} \Big\}.
\end{split}
\]
For a subset $A \subseteq V$, the marginal independence model on $A$, denoted by $I(\mathcal{G})_A$, is obtained by projecting $I(\mathcal{G})$ onto $A$:
\[
I(\mathcal{G})_A = \Big\{ (X,Y,Z) \in I(\mathcal{G}) \,\Big|\,  \;
 X, Y, Z \subseteq A   \Big\}.
\]
In general, \( I(\mathcal{G}_A) \neq I(\mathcal{G})_A \), since the class of DAGs is not closed under marginalization. Nevertheless, the inclusion $I(\mathcal{G})_A \subseteq I(\mathcal{G}_A)$ always holds.

\subsection{Bayesian Network and Marginal Distribution Model}

Let $X_V$ be a random vector indexed by $V$. Its state space is $\mathcal{X}_V \equiv \otimes_{v \in V} \mathcal{X}_v$, where each $\mathcal{X}_v$ is either a finite-dimensional Euclidean space or a finite discrete set. For a subset $A \subseteq V$, $X_A$ denotes the sub-vector of $X_V$ indexed by $A$, with range $\mathcal{X}_A \equiv \otimes_{v \in A} \mathcal{X}_v$. A probability distribution $P$ on $X_V$ is \emph{compatible} with $\mathcal{G}$ or equivalently factorizes according to $\mathcal{G}$, if 
\begin{equation}\label{eq-2}
	P(X_V) = \prod_{v \in V} P\left(X_v \mid X_{pa_{\mathcal{G}}(v)}\right).
\end{equation} 
That is, the conditional distribution of each $X_v$ depends only on its parents in $G$.

For a distribution $P$, we use $X_A \indep X_B \mid X_C [P]$ to indicate that $X_A$ is conditionally independent of $X_B$ given $X_C$ under $P$. If $P$ is compatible with $\mathcal{G}$, then for pairwise disjoint $A,B,C\subseteq V$, the graph separation $A \indep B \mid C [\mathcal{G}]$ implies $X_A \indep X_B \mid X_C [P]$.

A Bayesian network  $\mathcal{M}(\mathcal{G})$ is the set of probability distributions that are Markov with respect to the DAG $\mathcal{G}$. For a BN $\mathcal{M}(\mathcal{G})$ and a subset $A\subseteq V$, define the marginal model on $A$ by
$$
\mathcal{M}(\mathcal{G},A)=\{P_A(x_A)\mid P_A(x_A)=\textstyle\int_{\mathcal{X}_{V\backslash A}} dP(x),\, P\in \mathcal{M}(\mathcal{G})\}.
$$
We say that $\mathcal{M}(\mathcal{G})$ is \emph{collapsible} onto $A$ if and only if $\mathcal{M}(\mathcal{G}_A)=\mathcal{M}(\mathcal{G},A)$.  For ease of reference, Table~\ref{tab-notation} presents a list of commonly used symbols, systematically summarizing the main notations employed throughout the paper.

\section{Decomposition of Bayesian Networks}\label{sec-3}

\subsection{Directed Convex Subgraph}\label{sec-3-con}

For the decomposition to be valid, each resulting subgraph $\mathcal{G}_A$ must satisfy $\mathcal{M}(\mathcal{G}_A) = \mathcal{M}(\mathcal{G},A)$. That is, $\mathcal{M}(\mathcal{G})$ must be collapsible onto $A$. To reconstruct the joint distribution consistently, the intersections of decomposed subgraphs, namely the d-decomposers, must also be collapsible. Although verifying collapsibility is generally non-trivial, Heng et al. \cite{heng2026structural} showed that if a subgraph $\mathcal{G}_A$ is directed convex, then $\mathcal{M}(\mathcal{G})$ is collapsible onto $A \subseteq V$. For completeness, the formal definition of d-convex subgraphs is described as follows:

\begin{definition}[d-convex, \cite{heng2026structural}]\label{def-d-convexity}
	Let $\mathcal{G}=(V,E)$ be a DAG and let $A\subseteq V$. The induced subgraph $\mathcal{G}_A$ is \textbf{directed convex} (d-convex for short) if, for every pair of non-adjacent vertices $u$ and $v$ in $A$, there is no path $l_{uv}$ between $u$ and $v$ such that
	$
	A \cap V^o(l_{uv}) \subseteq V^c(l_{uv}) \subseteq An_{\mathcal{G}}(\{u,v\}).
	$
\end{definition}

Any path $l_{xy}$ satisfying the condition in Definition~\ref{def-d-convexity} is called an inducing path, or information path, over $A$ between $x$ and $y$. The term “information path” emphasizes that, if such a path exists between two non-adjacent vertices $x$ and $y$ in $A$, then no subset of $A \setminus \{x,y\}$ can d-separate them \cite{2000Cau}. Consequently, computing the marginal distribution on $A$ using $\mathcal{G}_A$ would introduce spurious conditional independencies. A simple example of a d-convex subgraph is provided below.

\begin{figure*}[!t]
    \centering
    \begin{minipage}[t]{0.3\textwidth}
    \vspace{0pt}
        \centering
        \begin{tikzpicture}[scale=0.8, transform shape,
            every node/.style={circle, draw=black, minimum size=7mm, inner sep=0pt, font=\large}]
            \node[fill=lightgray] (c) at (0,0.3){$c$};
            \node[fill=lightgray] (d) at (0,-1){d};
            \node[fill=lightgray] (i) at (2,-1){$i$};
            \node[fill=lightgray] (g) at (1,-2.2){$g$};
            \node[fill=lightgray] (s) at (2.2,-2.2){$s$};
            \node[fill=lightgray] (l) at (1,-3.5){$l$};
            \node[fill=lightgray] (j) at (2.2,-3.5){$j$};
            \node[fill=lightgray] (h) at (0,-3.5){$h$};
            \node[fill=lightgray] (k) at (0,-2.2){$k$};
            \node[fill=lightgray] (m) at (-1.3,-2.2){$m$};
            \draw[->,ultra thick] (c)--(d); \draw[->,ultra thick] (d)--(g);
            \draw[->,ultra thick] (i)--(g); \draw[->,ultra thick] (i)--(s);
            \draw[->,ultra thick] (g)--(l); \draw[->,ultra thick] (g)--(h);
            \draw[->,ultra thick] (l)--(j); \draw[->,ultra thick] (s)--(j);
            \draw[->,ultra thick] (k)--(h); \draw[->,ultra thick] (m)--(k);
        \end{tikzpicture}
        \caption*{(a)}
    \end{minipage}
    \begin{minipage}[t]{0.3\textwidth}
    \vspace{0pt}
        \centering
        \begin{tikzpicture}[scale=0.8]
            \node[scale=1.3,shape=ellipse, fill=lightgray] (a) at (0,-1.5){\scriptsize{$c,d$}};
            \node[scale=1.3,shape=ellipse, fill=lightgray] (b) at (2.5,-1.5){\scriptsize{$g,i,d$}};
            \node[scale=1.3,shape=ellipse, fill=lightgray] (c) at (2.5,-3.2){\scriptsize{$g,i,s$}};
            \node[scale=1.3,shape=ellipse, fill=lightgray] (d) at (2.5,-4.9){\scriptsize{$g,s,l$}};
            \node[scale=1.3,shape=ellipse, fill=lightgray] (f) at (0,-3.2){\scriptsize{$s,l,j$}};
            \node[scale=1.3,shape=ellipse, fill=lightgray] (g) at (0,-4.9){\scriptsize{$g,h,k$}};
            \node[scale=1.3,shape=ellipse, fill=lightgray] (i) at (-2.5,-4.9){\scriptsize{$m,k$}};
            \draw[-,ultra thick] (a)--(b); \draw[-,ultra thick] (b)--(c);
            \draw[-,ultra thick] (c)--(d); \draw[-,ultra thick] (d)--(f);
            \draw[-,ultra thick] (d)--(g); \draw[-,ultra thick] (i)--(g);
            \node at (1.1,-1.2){$\{d\}$};
            \node at (3.1,-4){$\{g,s\}$};
            \node at (1.2,-4.4){$\{g\}$};
            \node at (3.1,-2.4){$\{g,i\}$};
            \node at (1.5,-3.7){$\{s,l\}$};
            \node at (-1.25,-4.4){$\{k\}$};
        \end{tikzpicture}
        \caption*{(b)}
    \end{minipage}
    \hspace{0.02\textwidth}  
    \begin{minipage}[t]{0.3\textwidth}
        \vspace{0pt}
        \centering
        \begin{tikzpicture}[scale=0.8]
            \node[scale=1.3,shape=ellipse,fill=lightgray] (a) at (0,-1.5){\scriptsize{$c,d$}};
            \node[scale=1.3,shape=ellipse, fill=lightgray] (b) at (3.5,-1.5){\scriptsize{$g,i,d$}};
            \node[scale=1.3,shape=ellipse, fill=lightgray] (c) at (3.5,-3.2){\scriptsize{$g,i,s,j,l$}};
            \node[scale=1.3,shape=ellipse, fill=lightgray] (d) at (0,-3.2){\scriptsize{$g,h,k$}};
            \node[scale=1.3,shape=ellipse, fill=lightgray] (e) at (0,-4.9){\scriptsize{$m,k$}};
            \draw[-,ultra thick] (a)--(b); \draw[-,ultra thick] (b)--(c);
            \draw[-,ultra thick] (c)--(d); \draw[-,ultra thick] (d)--(e);
            \node at (1.5,-1.2){$\{d\}$};
            \node at (1.5,-2.85){$\{g\}$};
            \node at (4.05,-2.4){$\{g,i\}$};
            \node at (0.35,-4){$\{k\}$};
        \end{tikzpicture}
        \caption*{(c)}
    \end{minipage}

    \caption{(a) A DAG $\mathcal{G}=(V,E)$. (b) A minimal d-separation tree of $\mathcal{G}$; (c) a minimal d-decomposition tree of $\mathcal{G}$, obtained by iteratively merging adjacent nodes whose intersections are non-convex.  See Examples \ref{example-1} and \ref{example-2} for details.}
    \label{fig-1}
\end{figure*}

\begin{example}\label{example-1}
Consider the DAG $\mathcal{G}=(V,E)$ shown in Fig.~\ref{fig-1}(a), and define the vertex sets $$
A_1=\{d,g,l,j\}, \qquad A_2=\{g,i,s,l,j\}.$$

For $A_1$, consider the path
\[
l_{dj}: d \rightarrow g \leftarrow i \rightarrow s \rightarrow j .
\]
Along this path, the set of internal vertices satisfies
\[
V^o(l_{dj}) \cap A_1 = \{g\} \subseteq V^c(l_{dj}),
\]
and all collider vertices belong to the ancestor set of $\{d,j\}$, that is,
\[
V^c(l_{dj}) \subseteq An_{\mathcal{G}}(\{d,j\}).
\]
By Definition~\ref{def-d-convexity}, $l_{dj}$ is therefore an inducing path. This implies that the induced subgraph $\mathcal{G}_{A_1}$ is not d-convex in $\mathcal{G}$.  

Indeed, no subset of $A_1 \setminus \{d,j\} = \{g,l\}$ d-separates d and $j$ in $\mathcal{G}$. However, in the induced subgraph $\mathcal{G}_{A_1}$, the subset $\{l\}$ does d-separate d and $j$. Consequently, the conditional independence relations encoded by $\mathcal{G}$ and $\mathcal{G}_{A_1}$ differ, that is, $I(\mathcal{G})_{A_1} \neq I(\mathcal{G}_{A_1})$.

In contrast, the induced subgraph $\mathcal{G}_{A_2}$ is d-convex in $\mathcal{G}$, since there exists no inducing path in $\mathcal{G}$ connecting any pair of non-adjacent vertices in $A_2$.
\end{example}

In Example~\ref{example-1}, we showed that a non-d-convex subgraph can encode incorrect conditional independencies. The following lemma states an important property of d-convex subgraphs.

\begin{lemma}[\cite{heng2026structural}]\label{lem:xie-geng}
Given a BN $\mathcal{M}(\mathcal{G})$ and a subset $A\subseteq V$, if $\mathcal{G}_A$ is d-convex in $\mathcal{G}$, then we have $\mathcal{M}(\mathcal{G}, A) = \mathcal{M}(\mathcal{G}_A)$.
\end{lemma}

Lemma~\ref{lem:xie-geng} states that, for a Bayesian network $\mathcal{M}(\mathcal{G})$, if the graph $\mathcal{G}$ can be decomposed into a collection of d-convex subgraphs $A_1, \dots, A_k$ such that 
$$
\mathcal{G} = \mathcal{G}_{A_1} \cup \cdots \cup \mathcal{G}_{A_k},
$$
then any distribution $P \in \mathcal{M}(\mathcal{G})$ can be exactly reconstructed from its marginals over the subgraphs $A_i$. As the conclusion of this lemma plays a central role in the development of the main results, its proof is provided in Supplementary Material~\ref{app_lem:1} for completeness.

The next lemma gives a practical criterion for checking whether a subgraph is d-convex.

\begin{lemma}\label{thm-3.2}  
Let $\mathcal{G}=(V, E)$ be a DAG, let $A = V \backslash B$ for some $B \subseteq V$, and let $u,v \in A$ be two non-adjacent vertices. Then the following statements are equivalent:  
\begin{itemize}  
	\item[(i)] there exists an inducing path between $u$ and $v$;  
	\item[(ii)] there exists a path $\rho_{uv}$ connecting $u$ and $v$ in $(\mathcal{G}_{An_\mathcal{G}(\{u,v\})})^m$ such that $V^o(\rho_{uv}) \subseteq B$.  
\end{itemize}  
\end{lemma}

\begin{proof}  
The implication (i) $\implies$ (ii) follows directly from the definition of moralization. Conversely, consider the construction of $(\mathcal{G}_{An_\mathcal{G}(\{u,v\})})^m$. Without loss of generality, let $\rho_{uv} = (u, u_1, u_2, \ldots, u_k, v)$, and suppose that $(u_i, u_{i+1})$ is the only moral edge. Then there exists a vertex $w$ such that $u_i \rightarrow w \leftarrow u_{i+1}$ in $\mathcal{G}$.  

We can thus construct a path  
\[
\ell_{uv} = (u, u_1, u_2, \ldots, u_i, w, u_{i+1}, \ldots, u_k, v).
\]  
By the construction of $(\mathcal{G}_{An_\mathcal{G}(\{u,v\})})^m$, it follows that  
\[
V^o(\ell_{uv}) \cap A \subseteq V^c(\ell_{uv}) \subseteq An_{\mathcal{G}}(\{u,v\}),
\]  
as required.  
\end{proof}

Based on Lemma~\ref{thm-3.2}, we propose the following algorithm for testing d-convexity:

\begin{algorithm}[H]
	\caption{D-convexity testing algorithm \emph{DCTest}$(\mathcal{G}, A)$}
	\label{algorithm-convexity}
	\begin{algorithmic}[1]
		\Require A connected DAG $\mathcal{G} = (V, E)$ and a subset $A \subseteq V$
		\Ensure Whether $A$ is d-convex in $\mathcal{G}$
		\State Initialize \parbox[t]{0.9\linewidth}{
        $\displaystyle
        \mathbb{S} = \big\{ \{u, v\} \mid u, v \in mb_\mathcal{G}(V\backslash A), 
        \text{$u$ and $v$ are}\\ \text{ non-adjacent in $\mathcal{G}$} \big\}
        $
        }
		\ForAll{$\{u,v\} \in \mathbb{S}$}
		\If{there exists a path $l_{uv}$ between $u$ and $v$ in $(\mathcal{G}_{An_\mathcal{G}(\{u,v\})})^m$ such that $V^o(l_{uv}) \subseteq V\backslash A$}
		\State \Return \textbf{False}
		\EndIf
		\EndFor
		\State \Return \textbf{True}
	\end{algorithmic}
\end{algorithm}

\begin{proposition}
The computational complexity of Algorithm~\ref{algorithm-convexity} is $O(m)$, where $m$ denotes the number of edges in $\mathcal{G}$.
\end{proposition}

\begin{proof}
The overall complexity of the algorithm is dominated by the procedures for identifying the Markov boundary and performing moralization, each with a worst-case time complexity of $O(m)$. All remaining steps do not exceed this bound.
\end{proof}

\subsection{D-Convex Subgraph-Based Decomposition}

We now formally define the decomposition of a Bayesian network.

\begin{definition}[Decomposition]\label{def-3.1}  
Let $\mathcal{M}(\mathcal{G})$ be a Bayesian network, and let $A, B, S \subseteq V$ be three mutually disjoint subsets. We say that $(A, B, S)$ forms a decomposition of $\mathcal{M}(\mathcal{G})$,  
if the following conditions hold:  

(i) $A \cup B \cup S = V$;  

(ii) $S$ d-separates $A$ and $B$ in $\mathcal{G}$;  

(iii) $S$ is d-convex in $\mathcal{G}$.  

Further, we say that $S$ is a directed decomposer (d-decomposer). If both $A$ and $B$ are non-empty, the decomposition $(A, B, S)$ is called a proper decomposition. Unless otherwise stated, all decompositions considered in this paper are assumed to be proper.  
\end{definition}

The key requirement is that the decomposer be a d-convex d-separator, which guarantees collapsibility of the resulting sub-models.  Below, we present several fundamental properties of such decompositions, which support the correctness of the procedure and the reconstruction of the joint distribution from lower-dimensional subgraphs. 

\begin{proposition}\label{lem-6}  
Let $\mathcal{M}(\mathcal{G})$ be a Bayesian network with a decomposition $(A, B, S)$. Then, for any $v \in V$, either  
\[
fa_\mathcal{G}(v) \subseteq A \cup S \quad \text{or} \quad fa_\mathcal{G}(v) \subseteq B \cup S.
\]  
\end{proposition}

\begin{proof}  
Suppose there exist two vertices $u, w \in pa_\mathcal{G}(v)$ such that $u \in A$ and $w \in B$. Since $A$ and $B$ are d-separated by $S$, it follows that $v \in S$. This would create a path $l: u \rightarrow v \leftarrow w$ connecting $A$ to $B$ through $v \in S$, contradicting the assumption that $S$ d-separates $A$ and $B$. Therefore, $fa_\mathcal{G}(v) \subseteq A \cup S$ or $fa_{\mathcal{G}}(v) \subseteq B \cup S$.
\end{proof}

The following proposition states that sub-models obtained from a decomposition along a d-decomposer are collapsible.

\begin{proposition}\label{pro-3.2}  
Let $\mathcal{M}(\mathcal{G})$ be a Bayesian network decomposed into $\mathcal{M}_{G_{A \cup S}}$ and $\mathcal{M}_{G_{B \cup S}}$ by a d-decomposer $S$. Then, for each $K \in \{A \cup S, B \cup S, S\}$, we have
\[
\mathcal{M}(\mathcal{G},K) = \mathcal{M}(\mathcal{G}_K).
\]  
\end{proposition}

\begin{proof}  
It suffices to show that $K$ is d-convex in $\mathcal{G}$ for each $K \in \{A \cup S, B \cup S, S\}$, according to Lemma~\ref{lem:xie-geng}. Suppose, for the sake of contradiction, that $A \cup S$ is not d-convex. Then there exists an inducing path $l_{uv}$ over $A \cup S$ with endpoints $u, v \in A \cup S$.

Since $S$ d-separates $A$ and $B$ in $\mathcal{G}$, it follows that $u, v \in S$. By the definition of an inducing path, $l_{uv}$ is then an inducing path of $S$, which contradicts the assumption that $S$ is d-convex. The argument for $B \cup S$ is analogous, and the d-convexity of $S$ follows directly from its definition.

\end{proof}

Proposition~\ref{pro-3.2} shows that each decomposed sub-model coincides with the corresponding marginal model of the original network. This allows parameter learning to be performed directly on each sub-network using local data, with the results stored for subsequent use. Consequently, inference efficiency is substantially improved, and computational costs are reduced. Building on this property, Theorem \ref{thm-1} formalizes the factorization of the network's joint distribution across the decomposition.

\begin{theorem}\label{thm-1}
Suppose that $\mathcal{M}(\mathcal{G})$ is a Bayesian network and 
that $(A,B,S)$ forms a decomposition of $\mathcal{M}(\mathcal{G})$.
Then, for any $P \in \mathcal{M}(\mathcal{G})$, we have  
\[
P(x_V)\, P_{\mathcal{G}_S}(x_S) = P_{\mathcal{G}_{A \cup S}}(x_{A \cup S}) \, P_{\mathcal{G}_{B \cup S}}(x_{B \cup S}),
\]  
where $P_{\mathcal{G}_{A \cup S}}(x_{A \cup S})$, $P_{\mathcal{G}_{B \cup S}}(x_{B \cup S})$, and $P_{\mathcal{G}_S}(x_S)$ factorize according to $\mathcal{G}_{A \cup S}$, $\mathcal{G}_{B \cup S}$, and $\mathcal{G}_S$, respectively.  

\end{theorem}

\begin{proof}
The subsets $A$ and $B$ are d-separated by $S$ in $\mathcal{G}$. Hence, for any $P\in\mathcal{M}(\mathcal{G})$, we have $A\indep B\mid S$ under $P$, which implies
\begin{equation}\label{eq:factor1}
P(x_V)\,P_S(x_S)
=
P_{A\cup S}(x_{A\cup S})\,P_{B\cup S}(x_{B\cup S}).
\end{equation}

By Proposition~\ref{pro-3.2}, for each
$K\in\{A\cup S,\,B\cup S,\,S\}$, the marginal distribution $P_K$ belongs to
$ \mathcal{M}(\mathcal{G}_K)$ and therefore factorizes according to the induced
subgraph $\mathcal{G}_K$. Identifying each marginal distribution with its
representation in the corresponding sub-model, we may write
$P_K = P_{\mathcal{G}_K}$ for all such $K$.

Substituting this identification into \eqref{eq:factor1}, we obtain
\[
P(x_V)\,P_{\mathcal{G}_S}(x_S)
=
P_{\mathcal{G}_{A\cup S}}(x_{A\cup S})\,P_{\mathcal{G}_{B\cup S}}(x_{B\cup S}),
\]
which completes the proof.
\end{proof}

According to Theorem~\ref{thm-1}, the decomposition method proposed in this study offers several key advantages. First, the marginal distributions of local sub-models can be computed directly from subgraph data and stored, with collapsibility ensuring their correctness. Second, the decomposition introduces no additional directed edges, preserving the original subgraph topology and further enhancing both inference efficiency and accuracy.

In practice, accurately identifying suitable d-decomposers and efficiently decomposing BNs remain challenging. This difficulty arises from the potentially large number of d-separators in a DAG, which makes sequential d-convexity verification and decomposition computationally expensive. In the following subsection, we address this challenge by using a tree-based structure to efficiently construct a minimal d-decomposition tree, thereby enabling scalable parameter learning and probabilistic inference.

\subsection{An Efficient Decomposition Algorithm}

We first recall the tree concepts needed for the decomposition algorithm. A tree is an acyclic, connected graph in which a unique path exists between any two nodes. A tree $T = (\mathcal{C}, \mathcal{E}_T)$ satisfies the \emph{junction property} if, for any distinct nodes $C_i, C_j \in \mathcal{C}$, every node along the unique path connecting $C_i$ and $C_j$ in $T$ contains the intersection $C_i \cap C_j$. Furthermore, $T$ is called \emph{reduced} if, for any distinct nodes $C_i, C_j \in \mathcal{C}$, neither node is a subset of the other, i.e., $C_i \nsubseteq C_j$ and $C_j \nsubseteq C_i$.

A d-decomposer $S$ is called a \emph{minimal d-decomposer} if it is also a minimal d-separator; that is, there exist two vertices $x$ and $y$ such that $S$ is a minimal $xy$-d-separator. Based on minimal d-decomposers, a more refined \emph{minimal d-decomposition tree} can be constructed. We now provide its formal definition.

\begin{definition}[Minimal d-decomposition tree]
Given a DAG $\mathcal{G} = (V, E)$ and a reduced tree $T = (\mathcal{C}, \mathcal{E}_T)$ defined over $\mathcal{G}$, the tree $T$ is called a \textbf{minimal d-decomposition tree} of $\mathcal{G}$ if it satisfies the following conditions:

\begin{itemize}
\item For every edge $(C_i, C_j) \in \mathcal{E}_T$, the intersection $C_i \cap C_j$ is a minimal d-decomposer in $\mathcal{G}$;
\item Each node $C_i \in \mathcal{C}$ cannot be further decomposed by any minimal d-decomposer of $\mathcal{G}$.
\end{itemize}
\end{definition}

This definition specifies both the separator condition on edges and the irreducibility condition on nodes.  In general, a minimal d-decomposition tree satisfying these conditions is not unique.

Condition~(i) requires that, for each edge of the tree, the intersection of the corresponding nodes is a minimal d-decomposer in $\mathcal{G}$. This condition ensures that the Bayesian network $\mathcal{M}(\mathcal{G})$ is correctly decomposed into two sub-models along the edge.  

Condition~(ii) further requires that no minimal $d$-decomposer can decompose the sub-model associated with any tree node. This guarantees that each sub-model is structurally irreducible.

In practice, a minimal d-decomposition tree of a DAG can be constructed from a minimal d-separator tree, a concept introduced by Liu et al. \cite{liu2010minimal} for structure learning. During construction, adjacent nodes whose intersection forms a non-convex d-separator are iteratively merged until no further merges are possible. This process produces a minimal d-decomposition tree of the DAG. Based on this approach, we propose an algorithm for constructing a minimal d-decomposition tree of a DAG $\mathcal{G}$, as described in Algorithm~\ref{alg-main}, followed by a brief discussion of its correctness and computational complexity.

\begin{algorithm}[t]
    \caption{Construction of a Minimal d-Decomposition Tree}
    \label{alg-main}
    \begin{algorithmic}[1]
        \Require A Bayesian network $\mathcal{M}(\mathcal{G}) = (\mathcal{G}, \mathcal{P}(\mathcal{G}))$ 
        \Ensure A minimal d-decomposition tree of $\mathcal{M}(\mathcal{G})$
        
        \State Construct a minimal d-separator tree $T = (\mathcal{C}, \mathcal{E}_T)$ of $\mathcal{G}$
        \State Compute the set of candidate d-separators $\mathcal{S} = \{S_i = C_i \cap C_j : (C_i, C_j) \in \mathcal{E}_T\}$
        \ForAll{$S_i \in \mathcal{S}$}
            \If{\emph{DCTest}($\mathcal{G}, S_i$) is false}
                \State Remove the edge $(C_i, C_j)$ corresponding to $S_i$ from $T$
                \State Merge $C_i$ and $C_j$ into a single cluster
            \EndIf
        \EndFor
        \State \Return $T$
    \end{algorithmic}
\end{algorithm}

\begin{theorem}  
Let $\mathcal{M}(\mathcal{G})$ be a Bayesian network, and let $T = (\mathcal{C}, \mathcal{E}_T)$ be the tree constructed by Algorithm~\ref{alg-main}. Then $T$ is a minimal d-decomposition tree of $\mathcal{G}$.  
\end{theorem}  

\begin{proof}  
It suffices to show that, for every edge $(C_i, C_j) \in \mathcal{E}_T$, the intersection $C_i \cap C_j$ constitutes a minimal d-decomposer in $\mathcal{G}$, and that each cluster $C_i \in \mathcal{C}$ admits no further decomposition by any minimal d-decomposer of $\mathcal{G}$.  

Let $T_{\text{pre}}$ denote the minimal d-separator tree before merging. For each edge $(C_i, C_j) \in \mathcal{E}_T$, its intersection $C_i \cap C_j$ corresponds to some intersection $C_p \cap C_q$ of an edge $(C_p, C_q) \in \mathcal{E}_{T_{\text{pre}}}$. Since all intersections $C_p \cap C_q$ in $T_{\text{pre}}$ are minimal d-separators, the intersections $C_i \cap C_j$ in $T$ remain minimal. Furthermore, Algorithm~\ref{alg-main} guarantees that every $C_i \cap C_j$ in $T$ satisfies the d-convexity condition. Therefore, each intersection $C_i \cap C_j$ corresponding to an edge in $T$ is a minimal d-decomposer of $\mathcal{G}$.

Finally, each cluster $C_i \in \mathcal{C}$ admits no further decomposition by any minimal d-decomposer of $\mathcal{G}$.  
\end{proof}

\begin{theorem}  
The time complexity of Algorithm~\ref{alg-main} is at most $O(nm)$, where $n$ and $m$ denote the numbers of vertices and edges in $\mathcal{G}$, respectively.  
\end{theorem}  

\begin{proof}  
The worst-case time complexity is dominated by the construction of the minimal d-separator tree, which requires $O(nm)$ time \cite{liu2010minimal}. The remaining steps of the algorithm do not exceed this bound.  
\end{proof}

As an illustration, we consider the DAG $\mathcal{G}$ in Fig.~\ref{fig-1}(a) and apply Algorithm~\ref{alg-main} to construct its minimal d-decomposition tree.

\begin{example}\label{example-2}  
We first construct a minimal d-separator tree of $\mathcal{G}$ using the method of Liu et al. \cite{liu2010minimal}, as shown in Fig.~\ref{fig-1}(b). The set of all minimal d-separators is  
\[
\mathcal{S} = \bigl\{\{d\}, \{g\}, \{g,i\}, \{g,s\}, \{s,l\}, \{k\}\bigr\}.
\]

Applying the DCTest algorithm reveals that the d-separators $\{g,s\}$ and $\{s,l\}$ are not d-convex in $\mathcal{G}$. Consequently, the clusters $\{g,i,s\}$, $\{g,s,l\}$, and $\{s,l,j\}$ are merged into a single cluster $\{g,i,s,j,l\}$. The resulting minimal d-decomposition tree of $\mathcal{G}$ is shown in Fig.~\ref{fig-1}(c).  
\end{example}

\begin{corollary}\label{cor-1}
Let $\mathcal{M}(\mathcal{G})$ be a Bayesian network, and let 
$T = (\mathcal{C}, \mathcal{E}_T)$ denote a minimal d-decomposition tree of $\mathcal{G}$, where 
$\mathcal{C} = \{C_1, C_2, \ldots, C_k\}$ is the collection of clusters. 
Then any joint distribution $P(X_V) \in \mathcal{M}(\mathcal{G})$ admits the factorization
\begin{align*}
P(X_V)
=
\frac{\prod_{i=1}^k P_{\mathcal{G}_{C_i}}(x_{C_i})}
{\prod_{(C_i, C_j) \in \mathcal{E}_T} P_{\mathcal{G}_{C_i \cap C_j}}(x_{C_i \cap C_j})}.
\end{align*}
\end{corollary}

\begin{proof}
The result follows directly from iteratively applying Theorem~\ref{thm-1} along the edges of the minimal d-decomposition tree $T$.
\end{proof}

From Corollary~\ref{cor-1}, if a Bayesian network can be decomposed into a union of $k$ d-convex subgraphs, its joint distribution is fully determined by the marginal distributions over these $k$ subgraphs, together with the $k-1$ marginal distributions over their pairwise intersections. This decomposition reduces both computational and storage costs and provides the foundation for efficient parameter learning and probabilistic inference using a minimal d-decomposition tree.

\section{Parallelized Learning and Local Inference Based on Decomposition}\label{sec-4}

This section presents two procedures. We first perform parallel parameter learning on the decomposed subnetworks. We then develop a local inference algorithm based on pruning in a minimal d-decomposition tree. 

\subsection{Parallel Parameter Learning}

Parallel parameter learning improves scalability in high-dimensional Bayesian networks by splitting the global task into smaller subproblems. Using a minimal d-decomposition tree, the global estimation task is divided into independent subproblems corresponding to each cluster and their pairwise intersections. For each subproblem, we extract the corresponding data and estimate the local parameters by maximum likelihood or Bayesian methods. The global joint distribution is then reconstructed using Corollary~\ref{cor-1}.
The detailed procedure is summarized in Algorithm~\ref{alg:distributed-inference}.

\begin{algorithm}[htbp]
    \caption{Parallel Parameter Estimation via a Minimal d-Decomposition Tree}
    \label{alg:distributed-inference}
    \begin{algorithmic}[1]
        \Require A connected DAG $\mathcal{G}=(V,E)$ and a dataset $\mathcal{D}$
        \Ensure The estimated global model $\hat{P}(X_V)$
        
        \State Construct a minimal d-decomposition tree $T=(\mathcal{C}, \mathcal{E}_T)$ of $\mathcal{G}$ using Algorithm~\ref{alg-main}.
        
        \ForAll{$C \in \mathcal{C} \cup \{C_i \cap C_j \mid (C_i, C_j) \in \mathcal{E}_T\}$} \textbf{in parallel}
            \State Extract the induced subgraph $\mathcal{G}_C$ over variables $X_C$
            \State Extract the corresponding sub-dataset $\mathcal{D}_C$ from $\mathcal{D}$
            \State Estimate the local distribution $P(X_C)$ from $\mathcal{D}_C$ by maximum likelihood or Bayesian methods
        \EndFor
        
        \State Reconstruct the estimated global model $\hat{P}(X_V)$ using the factorization in Corollary~\ref{cor-1}
        \State \Return $\hat{P}(X_V)$
    \end{algorithmic}
\end{algorithm}

\subsection{Local Statistical Inference}

The goal of probabilistic inference in large-scale BNs is to compute distributions over specified target subsets efficiently. A minimal d-decomposition tree typically contains numerous nodes that are irrelevant to the query variables, and directly using the entire tree incurs unnecessary computational overhead. To address this, we introduce recursive reduction rules that prune irrelevant leaves from a minimal d-decomposition tree, thereby improving both the efficiency and scalability of inference.

\begin{proposition}\label{pro-3}
Let $C_i$ be a leaf node of $T$, and let $V' = \bigcup_{C' \in \mathcal{C} \setminus \{C_i\}} C'$. If both the query variables and the evidence variables are contained in $V'$, then $C_i$ can be safely removed from $T$.
\end{proposition}

\begin{proof}
Let $T'$ denote the minimal d-decomposition tree obtained after removing $C_i$. Since $V'$ contains all query and evidence variables, and remains d-convex in $\mathcal{G}$, the model $\mathcal{M}(\mathcal{G})$ can be collapsed onto the sub-model $\mathcal{M}(\mathcal{G}_{V'})$ according to Theorem~\ref{thm-1}. Consequently, the target posterior distribution can be computed entirely from $\mathcal{M}(\mathcal{G}_{V'})$, and removing $C_i$ does not change the inference result.
\end{proof}

Proposition~\ref{pro-3} provides a criterion for pruning leaf nodes from a minimal d-decomposition tree without affecting the accuracy of inference on the query variables. The key idea is that any information relevant to the query or evidence variables in such leaf nodes is already contained in the remaining nodes, and the model $\mathcal{M}(\mathcal{G})$ is collapsible onto this remaining set. By iteratively removing such leaf nodes, a minimal d-decomposition tree achieves structural reduction while preserving all dependencies required for inference. This property ensures that local posterior probabilities can be computed efficiently without processing the entire network, as formalized in the following algorithm.

\begin{algorithm}[htbp]
	\caption{Local Inference via Pruning on a Minimal d-Decomposition Tree}
	\label{alg:local-inference}
	\begin{algorithmic}[1]
		\Require 
		A minimal d-decomposition tree $T = (\mathcal{C}, \mathcal{E}_T)$ with learned parameters, a query variable set $Q$, an evidence variable set $E$, and an observed value vector $x_E$.
		\Ensure The posterior distribution $\hat{P}(X_Q \mid X_E = x_E)$
		
		\Repeat
			\For{each leaf node $C_i$ in $T$}
				\If{$(Q \cup E) \cap C_i \subseteq \bigcup_{C' \in \mathcal{C} \setminus \{C_i\}} C'$}
					\State Remove $C_i$ from $T$
				\EndIf
			\EndFor
		\Until{no leaf nodes can be further removed}
		
		\If{$T$ contains only one remaining cluster $C$}
			\State Apply variable elimination in the sub-model $\mathcal{M}(\mathcal{G}_C)$ to compute $\hat{P}(X_Q \mid X_E = x_E)$
		\Else
			\State Apply belief propagation on the pruned tree $T$ to compute $\hat{P}(X_Q \mid X_E = x_E)$
		\EndIf
		
		\State \Return $\hat{P}(X_Q \mid X_E = x_E)$
	\end{algorithmic}
\end{algorithm}

Algorithm~\ref{alg:local-inference} performs local inference on a Bayesian network using a minimal d-decomposition tree. The procedure removes leaf nodes whose query or evidence variables are already included in the remaining clusters. This pruning continues until all remaining leaves are relevant to the query or evidence. If only one cluster remains, variable elimination is applied to the corresponding sub-model. Otherwise, belief propagation is applied to the pruned minimal d-decomposition tree to compute the posterior probability of the query variables given the evidence. This yields exact inference while avoiding computation on parts of the network that are irrelevant to the target posterior.

\section{Empirical studies}\label{sec-5}

In this section, we report numerical experiments evaluating parameter learning, model accuracy, and inference efficiency under the proposed minimal d-decomposition tree framework. Experiments are conducted on a system with Intel(R) Xeon(R) Silver 4215R CPUs (2 processors) and 128 GiB of memory. All code used in this study is available at \url{https://github.com/Balance-H/Decomposition-for-BNs}.

\subsection{Parameter Learning}\label{sec-5.1}

We first compare parameter estimation on decomposed subnetworks with direct global estimation. Six representative Bayesian networks from the \texttt{BNlearn} repository are used—\textit{Child}, \textit{Alarm}, \textit{Hailfinder}, \textit{Hepar2}, \textit{Win95pts}, and \textit{Pigs}—with vertex counts ranging from 20 to 441. Experiments are conducted as follows:

\begin{itemize}
\item[(1)] For each network $\mathcal{M}(\mathcal{G})$, generate a dataset $\mathcal{D}$ using forward sampling.
\item[(2)] Estimate the parameters of the Bayesian network defined on $\mathcal{G}$ from $\mathcal{D}$, and record the computational time as $t_{\text{full}}$.
\item[(3)] Decompose $\mathcal{M}(\mathcal{G})$ into multiple sub-models and estimate their local parameters in parallel using the corresponding subsets of $\mathcal{D}$. Record the time required as $t_{\text{decom}}$.
\item[(4)] Compute the ratio of computational times, $t_{\text{full}} / t_{\text{decom}}$, as the evaluation metric. Repeat this procedure 100 times for each network to obtain 100 time ratios, which are then visualized using a scatter plot.
\end{itemize}

\begin{remark}
  In these Bayesian networks, all variables are binary. Parameters are estimated by maximum likelihood with a fixed sample size of 150,000. The number of parallel cores in the decomposition method is assigned dynamically based on the number of sub-models, with a maximum of 9 cores. To ensure a fair comparison, the reported running times include only the parameter learning phase.
\end{remark}

\begin{figure}[!t]
	\centering
	\includegraphics[width=0.45\textwidth]{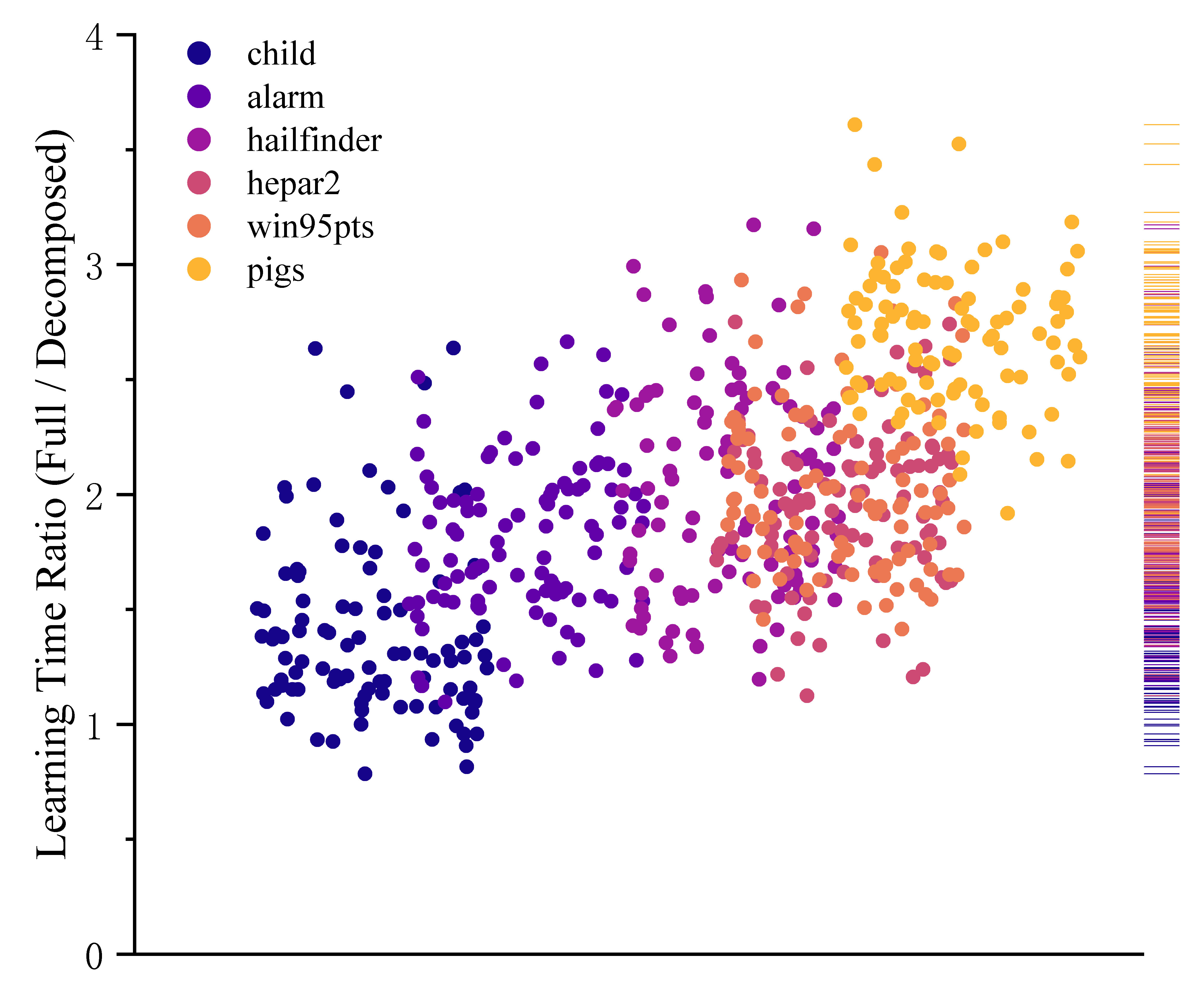}
	\caption{Efficiency of parameter estimation for global learning and parallel decomposition}
	\label{fig-5}
\end{figure}

Figure~\ref{fig-5} shows how the efficiency of parameter estimation changes with network size, with networks arranged from small to large. For smaller networks, such as \textit{Child} and \textit{Alarm}, the decomposition method provides limited efficiency gains due to the overhead of parallel computation. As the network size increases, the advantage of decomposition becomes more pronounced. For large networks, parallel estimation on sub-models yields substantial speedups over global estimation. Across the six benchmark networks, the serial portion accounts for less than 20\% of the total computation time on average. Detailed information on the treewidth of the networks and the distribution of subgraph dimensions after decomposition is provided in Supplementary Material~\ref{supp-ana}.

\subsection{Accuracy of Parallelized Learning}\label{sec-5.2}

In this subsection, we evaluate the accuracy of the learned model $\mathcal{M}'(\mathcal{G})$ by comparing it with the original model
$\mathcal{M}(\mathcal{G})$ using the Kullback–Leibler divergence, Hellinger distance, and Bhattacharyya distance.

\subsubsection{Multinomial Bayesian Networks}\label{sec-5.2.2}

We consider discrete Bayesian networks of varying sizes:
\begin{itemize}
\item \textbf{Small:} \textit{Asia}, \textit{Survey};
\item \textbf{Medium:} \textit{Child}, \textit{Insurance}, \textit{Mildew}, \textit{Alarm};
\item \textbf{Large:} \textit{Hepar2}, \textit{Hailfinder}, \textit{Win95pts}, \textit{Pathfinder}.
\end{itemize}

The experimental process is summarized as follows:
\begin{itemize}
	\item[(1)] For each network $\mathcal{M}(\mathcal{G})$, a compatible multinomial distribution with binary state space $\{0,1\}$ is generated at random. Datasets of sizes 500, 1000, 2500, 5000, 7500, and 10000 are then sampled using forward sampling.

    \item[(2)] Algorithms~\ref{alg-main} and~\ref{alg:distributed-inference} are then used to decompose $\mathcal{G}$ and estimate the resulting model $\mathcal{M}'(\mathcal{G})$.

	\item[(3)] The distance between $\mathcal{M}(\mathcal{G})$ and $\mathcal{M}^\prime(\mathcal{G})$ is quantified using three metrics:
	\begin{itemize}
		\item[(a)] \textbf{Kullback-Leibler divergence}:
		\[
		\mathrm{KL}(\hat{p}\,\|\,p) = \sum_{x} \hat{p}(x) \left[ \ln \hat{p}(x) - \ln p(x) \right].
		\]
		
		\item[(b)] \textbf{Hellinger distance}:
		\[
		\mathrm{H}^2(\hat{p}, p) = \frac{1}{2} \sum_{x} \left( \sqrt{\hat{p}(x)} - \sqrt{p(x)} \right)^2.
		\]
		
		\item[(c)] \textbf{Bhattacharyya distance}:
		\[
		\mathrm{BD}(\hat{p}, p) = -\ln \left( \sum_{x} \sqrt{\hat{p}(x) \cdot p(x)} \right).
		\]
	\end{itemize}
	
	\item[(4)] The above procedure is repeated 100 times for each network with independently generated distributions to evaluate variability.		
	
	\end{itemize}

For high-dimensional Bayesian networks, exact computation of distributional distances between models is often infeasible. Therefore, Gibbs sampling is used to approximate the three distances. The results are summarized in Fig.~\ref{fig-6}, where columns one to three correspond to the three distance metrics, and rows one to three represent networks of increasing size. In each subplot, the horizontal axis denotes the sample size, and the vertical axis shows the estimated distance. The results indicate that the model $\mathcal{M}'(\mathcal{G})$ obtained through parallel decomposition learning closely matches the distribution of the original model $\mathcal{M}(\mathcal{G})$. All three metrics decrease significantly as the sample size increases. When the sample size reaches 10,000, the three distances are close to zero, indicating that the proposed method recovers the target distributions with high accuracy.

\begin{figure*}[!t]
	\centering
	\includegraphics[width=\textwidth]{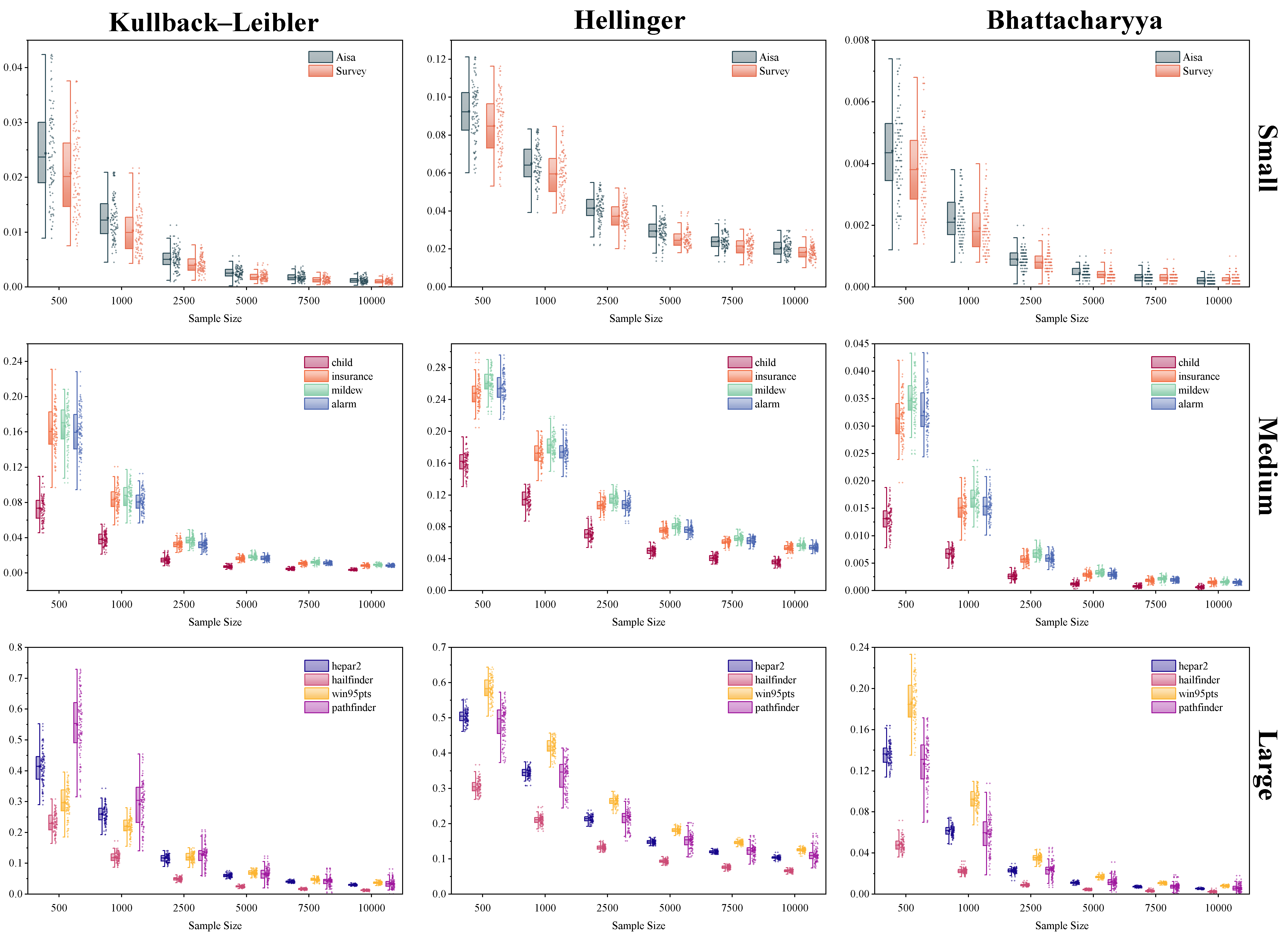}
	\caption{Box plots of distributional distances between the original and parallelized-learning models.}
	\label{fig-6}
\end{figure*}

\subsubsection{Gaussian Bayesian Networks}

We further examine four continuous Gaussian Bayesian networks from the \texttt{BNlearn} repository: \texttt{ecoli70}, \texttt{magic-niab}, \texttt{magic-irri}, and \texttt{arth150}, with variable dimensions ranging from 44 to 107. For each Bayesian network $\mathcal{M}(\mathcal{G})$, we follow the procedure in Section~\ref{sec-5.2.2} and evaluate the learned model using the following three metrics:

\begin{itemize}

\item[(a).] \textbf{Kullback–Leibler divergence}:
\begin{equation*}
\resizebox{0.95\linewidth}{!}{$
\mathrm{KL}(\hat{p} \,\|\, p)
=
\frac{1}{2}
\left[
\mathrm{tr}\!\left(\Sigma^{-1} \hat{\Sigma}\right)
+
(\mu - \hat{\mu})^\top \Sigma^{-1} (\mu - \hat{\mu})
-
n
+
\ln\!\left(
\frac{\det \Sigma}{\det \hat{\Sigma}}
\right)
\right].
$}
\end{equation*}

\item[(b).] \textbf{Hellinger distance}:
\begin{equation*}
\resizebox{0.95\linewidth}{!}{$
\mathrm{H}^2(\hat{p}, p)
=
1
-
\frac{|\Sigma|^{1/4} \, |\hat{\Sigma}|^{1/4}}
{\left|\frac{\Sigma + \hat{\Sigma}}{2}\right|^{1/2}}
\exp\!\left(
-\frac{1}{8}
(\mu - \hat{\mu})^\top
\left(\frac{\Sigma + \hat{\Sigma}}{2}\right)^{-1}
(\mu - \hat{\mu})
\right).
$}
\end{equation*}

\item[(c).] \textbf{Wasserstein-2 distance}:
\begin{equation*}
\resizebox{0.95\linewidth}{!}{$
\mathrm{W}_2^2(\hat{p}, p)
=
\|\mu - \hat{\mu}\|_2^2
+
\mathrm{tr}\!\left(
\Sigma + \hat{\Sigma}
-
2\left(
\Sigma^{1/2}
\hat{\Sigma}
\Sigma^{1/2}
\right)^{1/2}
\right).
$}
\end{equation*}

\end{itemize}

Here, $p = \mathcal{N}(\mu, \Sigma)$ and $\hat{p} = \mathcal{N}(\hat{\mu}, \hat{\Sigma})$ denote the true and estimated Gaussian distributions, respectively, and $n$ denotes the dimensionality of the network.

\begin{figure*}[!t]
	\centering
	\includegraphics[width=\textwidth]{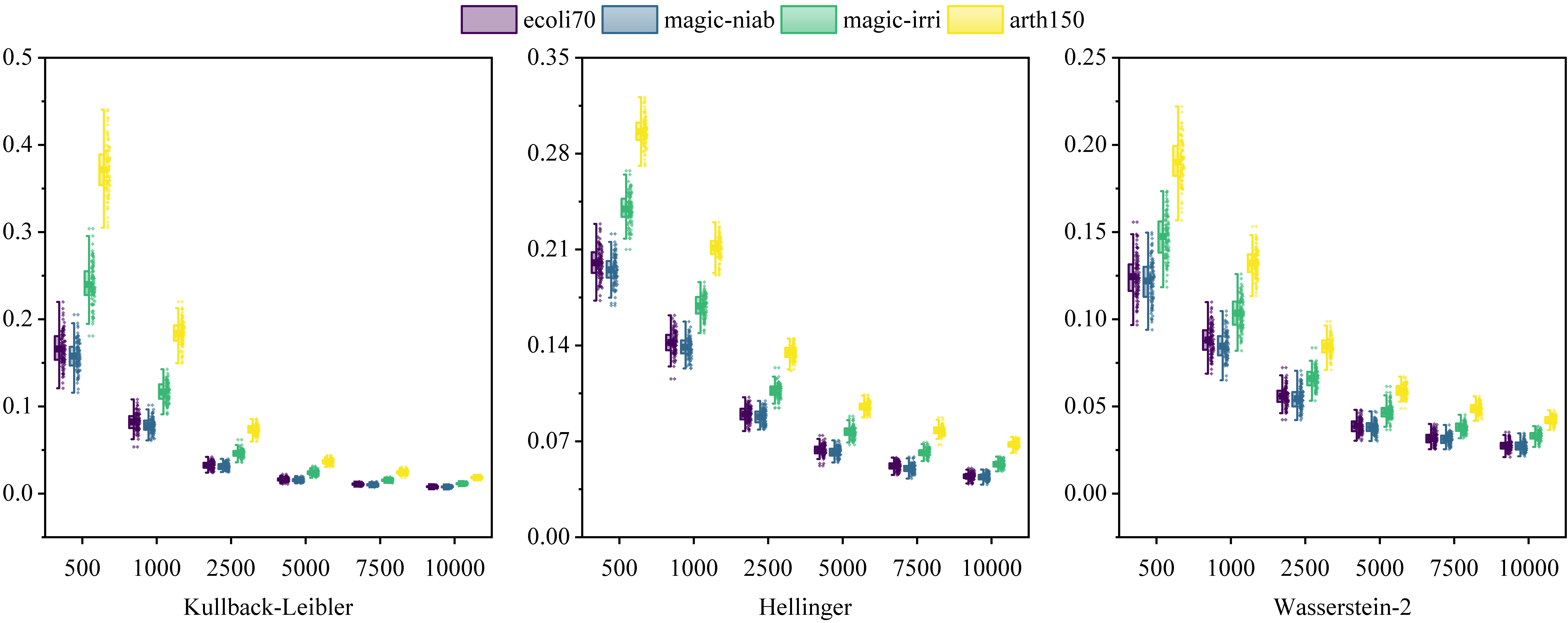}
	\caption{Box plots of distributional distances between the original model $\mathcal{M}(\mathcal{G})$ and the learned model $\mathcal{M}^\prime(\mathcal{G})$.}
	\label{fig:divergence_gaussian}
\end{figure*}

The results in Fig.~\ref{fig:divergence_gaussian} show that, for continuous Gaussian networks, the model learned via decomposition, $\mathcal{M}^\prime(\mathcal{G})$, closely approximates the original model $\mathcal{M}(\mathcal{G})$. The Wasserstein-2 distance stays low across all four networks and decreases as the sample size increases, stabilizing around 0.04 for 10,000 samples. This suggests that the learned model preserves the main statistical structure, including means and covariances. At the same time, the decomposition framework reduces computational cost. In addition, the average condition numbers of the estimated covariance matrices for these networks are reported in Supplementary Material~\ref{supp-con}, providing further evidence of the numerical stability of the learned Gaussian models.


\subsection{Local Inference in Minimal d-Decomposition Trees}
The previous experiments showed that the learned decomposition model closely matches the original model in distribution. We therefore turn to inference efficiency and evaluate the computational cost of joint probability queries under a minimal d-decomposition tree framework. Meanwhile, we conduct stress tests by gradually increasing the dimensionality of the query variables to assess the benefits of the pruning mechanism in high-dimensional settings.

The experiments are conducted on the DAG \textbf{Pathfinder}, denoted by $\mathcal{G}$. This network consists of 109 vertices and 195 directed edges. The overall experimental procedure is summarized as follows: 

\begin{enumerate}
	\item[(1)] \textbf{Data Generation:} For the DAG $\mathcal{G}$, compatible binary multinomial distributions are randomly generated to construct the Bayesian network $\mathcal{M}(\mathcal{G})$, from which 1,000 samples are drawn.  
	
	\item[(2)] \textbf{Inference Experiments:} A query set $Q$ is selected at random. For each such $Q$, we compute the corresponding joint probability using three tree-based inference methods: the proposed minimal d-decomposition tree method, Wu’s lossless decomposition method \cite{wu2007lossless}, and the conventional junction-tree method (Algorithm 4 in \cite{2009Probabilistic}). For each method, the inference time and the root mean squared error (RMSE) relative to the original model are recorded.
	
	\item[(3)] \textbf{Repetition and Statistics:} For different sizes of $Q$ (i.e., $|Q| \in \{2,5,10,15,20,25\}$), Step 2 is repeated 100 times. The average inference time and RMSE are reported to assess the stability and performance of each method.
\end{enumerate}

\begin{table}[htbp]
\centering
\caption{Average inference time and RMSE over 100 runs for three inference methods with varying query dimensions $|Q|$.} 
\begin{tabular}{ccccc}
\toprule
$|Q|$ & \textbf{Decomp. } & \textbf{Lossless} & \textbf{Belief Prop.} & \textbf{RMSE} \\
\midrule
2  & \textbf{0.0613} & 0.5061 & 0.7743 & 9.78E-04 \\
5  & \textbf{0.1176} & 0.4984 & 0.7668 & 7.80E-06 \\
10 & \textbf{0.1855} & 0.5112 & 0.7744 & 3.39E-08 \\
15 & \textbf{0.2338} & 0.5520 & 0.7729 & 2.25E-10 \\
20 & \textbf{0.2966} & 2.4314 & 0.8450 & 7.11E-13 \\
25 & \textbf{1.8073} & 30.1472 & 1.8515 & 1.39E-15 \\
\bottomrule
\end{tabular}
\label{tab:inference_time_rmse}
\end{table}

The experimental results are shown in Table~\ref{tab:inference_time_rmse}. All three methods are exact, so they have the same target inference output. We therefore report only one average RMSE column. The nonzero RMSE reflects the discrepancy between the learned model and the original data-generating model, rather than the differences among the three exact inference procedures.
The results indicate that pruning significantly improves inference efficiency for low-dimensional queries.  As the query dimension increases, the efficiency gains from pruning gradually decrease. When the query dimension reaches $|Q|=25$, the inference efficiency is nearly the same as that of the conventional junction tree. In addition, the computation time of the lossless decomposition algorithm increases sharply at $|Q|=25$. This increase has two main causes: the method is not specifically designed for small-query pruning, and larger query sets induce exponentially larger tensor products during marginalization.

\section{Conclusion and Discussion}\label{sec-6}

We proposed a decomposition framework for Bayesian networks that partitions the original network into subgraphs whose associated sub-models can be learned and stored in parallel. We also constructed a minimal d-decomposition tree together with a pruning rule for local inference. Based on this framework, we developed two algorithms for parameter estimation and probabilistic inference. The experiments show that the proposed approach improves computational efficiency in both parameter estimation and inference, especially when the query dimension is small.

A minimal d-decomposition tree is not unique, and identifying a decomposition that minimises computational cost for a given network remains an open problem. For subgraphs that remain high-dimensional after decomposition, one can apply a second-stage decomposition to build a nested d-decomposition tree, analogous to the nested junction tree in  \cite{kjaerulff1998inference}. Such extensions may further improve inference efficiency and merit future study.

The proposed method also has limitations. The method is less effective for dense networks and for high-dimensional queries. However, this limitation is shared more broadly by exact inference methods.

\section{Acknowledgement}

This work was partially supported by the National Natural Science Foundation of China
(Grant Numbers 11701491, 11726630), and Tianyuan Fund for Mathematics of the National Natural Science Foundation of China
(Grant Numbers 12426520 and 12426105).


\bibliographystyle{IEEEtran}  
\bibliography{ref}
\newpage
\section*{Biography Section}

\vspace{11pt}

\begin{IEEEbiography}[{\includegraphics[width=1in,height=1.25in,clip,keepaspectratio]{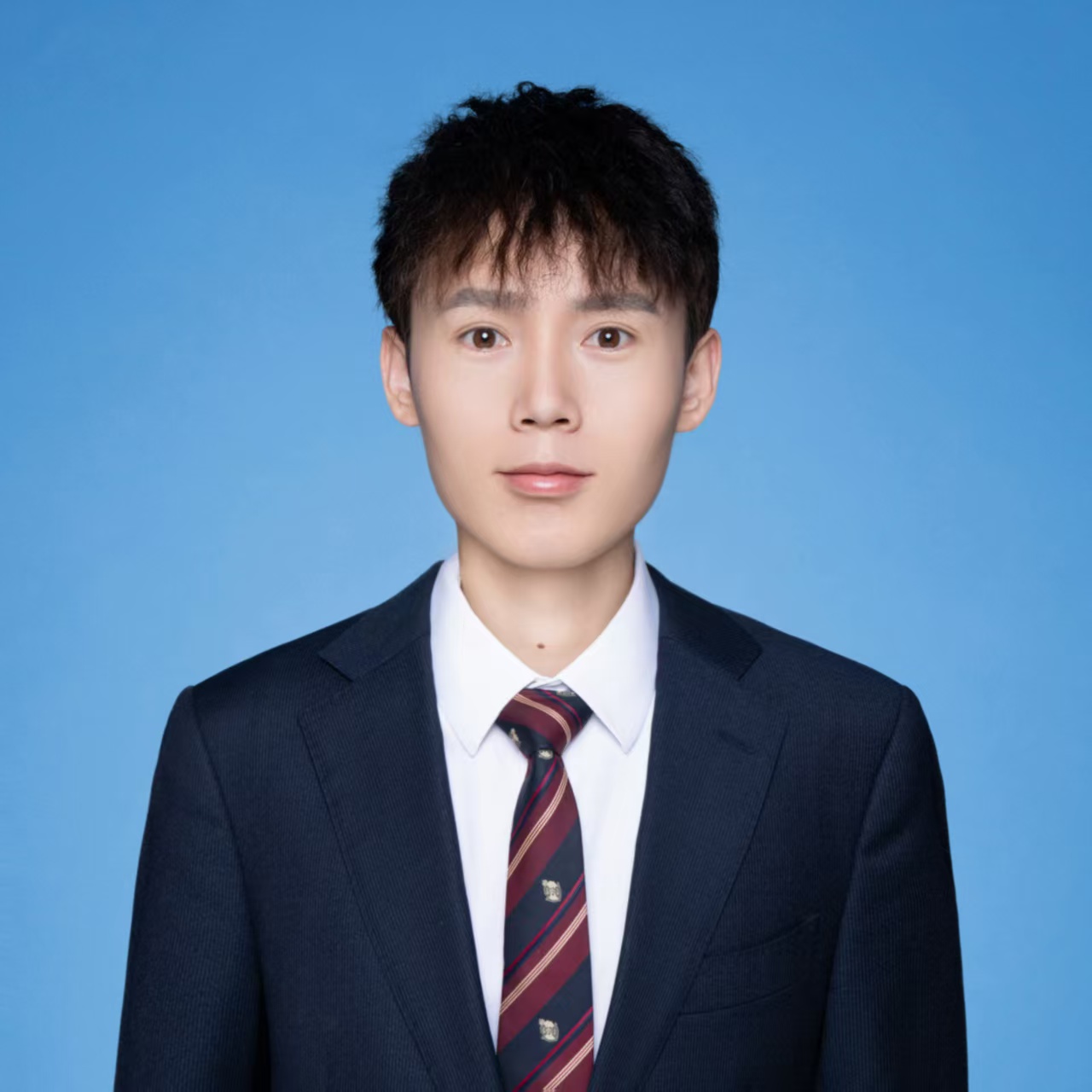}}]{Pei Heng}
is a Ph.D. candidate at Northeast Normal University. His research interests include structural dimensionality reduction and log-linear modelling.
\end{IEEEbiography}

\vspace{11pt}

\begin{IEEEbiography}[{\includegraphics[width=1in,height=1.25in,clip,keepaspectratio]{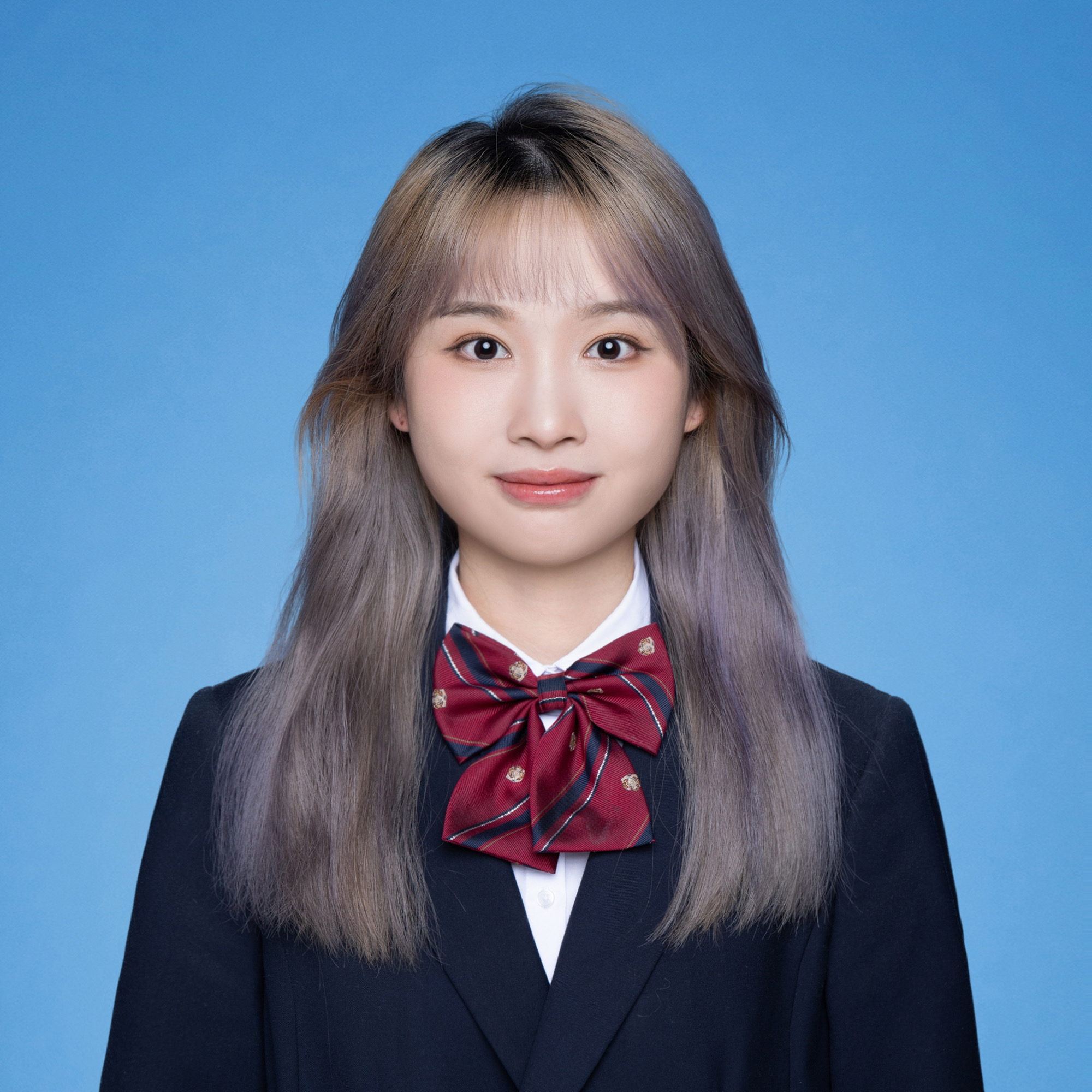}}]{Xinyi Hu}
is an M.S. candidate at Xinjiang University. Her research interests include Bayesian network dimensionality reduction and parallel inference.
\end{IEEEbiography}

\vspace{11pt}

\begin{IEEEbiography}[{\includegraphics[width=1in,height=1.25in,clip,keepaspectratio]{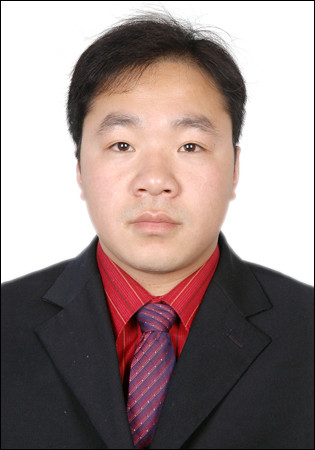}}]{Yi Sun} received the Ph.D. degree in mathematics from Nankai University in 2016 and is currently a Professor at Xinjiang University of Finance and Economics. His research interests include structural dimensionality reduction, causal inference, and graph decomposition.
\end{IEEEbiography}

\vfill

\newpage
\onecolumn 

\setcounter{section}{0}      
\setcounter{equation}{0}     

\renewcommand{\thesection}{S.\arabic{section}} 

\renewcommand{\theequation}{S.\arabic{equation}} 

\setcounter{theorem}{0}      
\renewcommand{\thetheorem}{S.\arabic{theorem}}

\setcounter{proposition}{0} 
\renewcommand{\theproposition}{S.\arabic{proposition}}

\setcounter{figure}{0} 
\renewcommand{\thefigure}{S.\arabic{figure}}

\setcounter{table}{0} 
\renewcommand{\thetable}{S.\arabic{table}}

\begin{center}
	{\Large\bfseries
		\linespread{1.3}\selectfont
		Supplementary Material for ``Decomposition for Bayesian Networks: Local and Parallel Inference''
	}
	
	\vspace{1.5em}

\end{center}

\section{Proof of Lemma 1}\label{app_lem:1}

\begin{proposition}[\cite{evans2016graphs}\label{prop-5}]
	Let $\mathcal{G}$ and $\mathcal{G}^{\prime}$ be $D A G s$ with the same vertex set $V$.
	
	(a) If $A \subseteq V$ is an ancestral set in $\mathcal{G}$, then $\mathcal{M}(\mathcal{G}, A)=\mathcal{M}\left(\mathcal{G}_A\right)$.
	
	(b) If $\mathcal{G}^{\prime} \subseteq \mathcal{G}$, then $\mathcal{M}\left(\mathcal{G}^{\prime}\right) \subseteq \mathcal{M}(\mathcal{G})$.
\end{proposition}

\begin{proof}[\textbf{Proof of Lemma  1}, \cite{heng2026structural}]
	First, let $P \in \mathcal{M}(\mathcal{G})$. Since $\mathcal{G}_A$ is d-convex in $\mathcal{G}$, we have $I(\mathcal{G})_A = I(\mathcal{G}_A),$
	so any conditional independence $X \indep Y \mid Z$ implied by $\mathcal{G}_A$ for disjoint $X,Y,Z \subseteq A$ also holds in $P$. Therefore, the marginal distribution $P_A$ satisfies all conditional independences of $\mathcal{G}_A$, implying $P_A \in \mathcal{M}(\mathcal{G}_A),$
	and hence
	\[
	\mathcal{M}(\mathcal{G}, A) \subseteq \mathcal{M}(\mathcal{G}_A).
	\]
	
	Conversely, consider the DAG $\mathcal{G}' = (V, E_A)$ obtained by treating $V \setminus A$ as isolated vertices in $\mathcal{G}_A$. Then $A$ is an ancestral set in $\mathcal{G}'$, and Proposition \ref{prop-5}(a) gives $\mathcal{M}(\mathcal{G}_A) = \mathcal{M}(\mathcal{G}', A).$
	Since $\mathcal{G}' \subseteq \mathcal{G}$, Proposition \ref{prop-5}(b) implies
	$\mathcal{M}(\mathcal{G}', A) \subseteq \mathcal{M}(\mathcal{G}, A),$
	and thus
	\[
	\mathcal{M}(\mathcal{G}_A) = \mathcal{M}(\mathcal{G}', A) \subseteq \mathcal{M}(\mathcal{G}, A).
	\]
	
	Combining both directions establishes the equality
	\[
	\mathcal{M}(\mathcal{G}, A) = \mathcal{M}(\mathcal{G}_A).
	\]
\end{proof}

\section{Analysis of Clique Size Distribution and Treewidth}\label{supp-ana}

In this section, we analyse the clique size distribution and treewidth of the networks used in Section~V experiments on parallelized parameter learning. Specifically, we consider six representative Bayesian networks: \textit{Child}, \textit{Alarm}, \textit{Hailfinder}, \textit{Hepar2}, \textit{Win95pts}, and \textit{Pigs}.

\subsection*{Normalized Clique Size Distribution}

For each network, we first construct the minimal d-decomposition tree. Let $|C|$ denote the size of a clique, that is, the number of variables in the decomposed subnetwork, and let $|V|$ denote the total number of variables in the network. To allow comparison across networks of different sizes, we normalise the clique size as $|C| / |V|$.

\begin{figure}[htbp]
	\centering
	\includegraphics[width=0.7\textwidth]{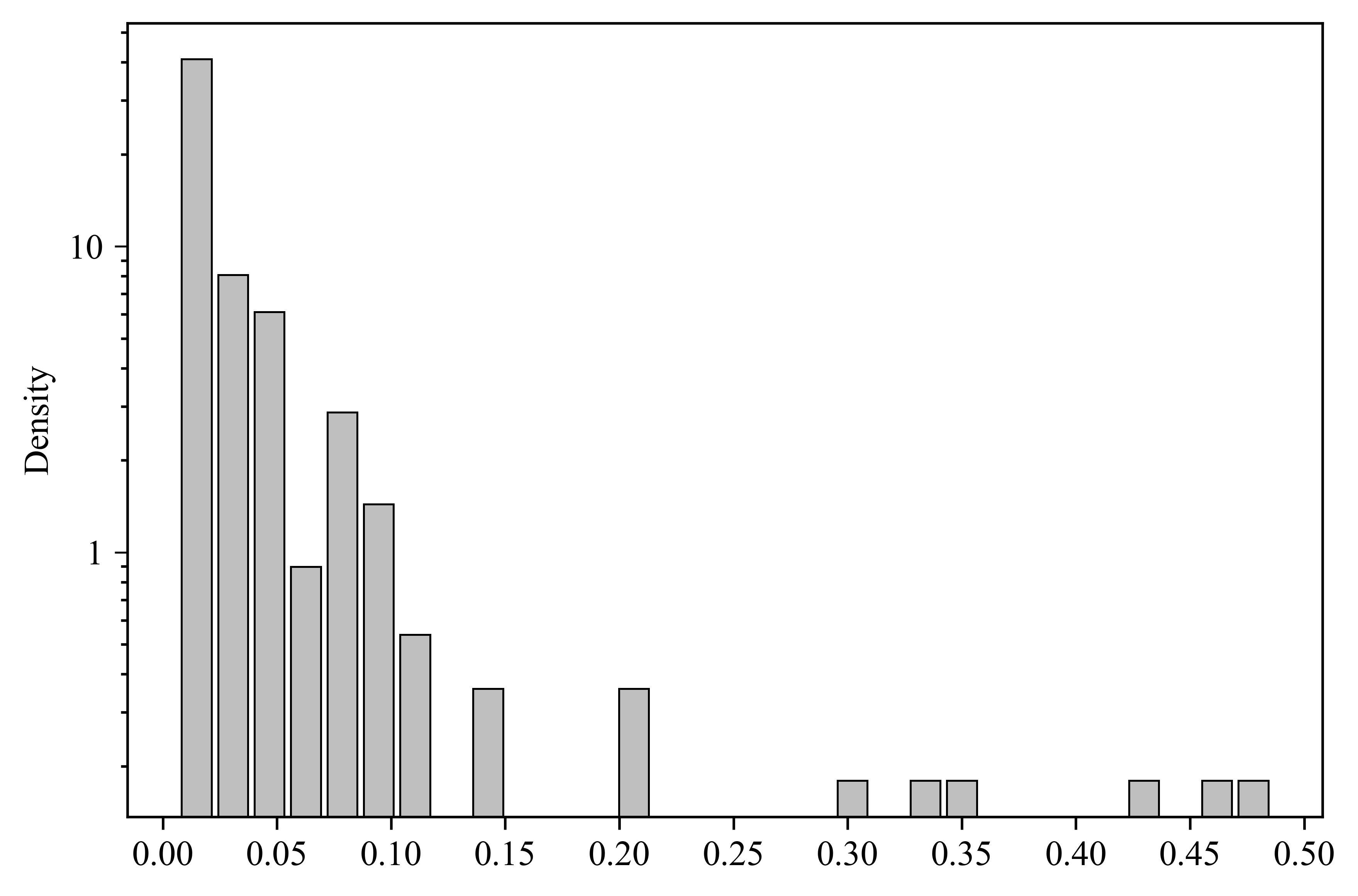}
	\caption{Histogram of normalized clique sizes $|C|/|V|$ across the six networks. The density axis is on a logarithmic scale to better illustrate the tail behaviour.}
	\label{fig:clique_hist}
\end{figure}

Figure~\ref{fig:clique_hist} shows a histogram of all normalized clique sizes aggregated across the six networks. The density axis is plotted on a logarithmic scale to better highlight the tail behaviour. The results indicate that over 90\% of the normalized cliques have a size smaller than 0.1, showing that most cliques in the decomposition are small relative to the full network. These small clique sizes directly contribute to efficient local computations during parameter learning. Only a small fraction of cliques are large, reflecting more complex substructures in some networks.

\subsection*{Treewidth Analysis}

Table~\ref{tab:treewidth_heuristics} reports the estimated treewidth of each network, computed using two standard elimination heuristics: \textit{minimum degree} and \textit{minimum fill-in} \cite{bodlaender2010treewidth}. The results show that all networks have relatively small treewidth, indicating that these real-world networks can be naturally decomposed into smaller subnetworks.  

This structural property allows parameter learning and probabilistic inference to be divided into largely independent subproblems. Consequently, parameters can be estimated and stored separately for each subnetwork. Table~\ref{tab:treewidth_heuristics} also reports the average speed-up ratios obtained by parallelized parameter learning compared with global parameter learning in Section~V-A. These results quantify the practical computational benefits of the proposed structural decomposition.

\begin{table}[htbp]
	\centering
	\caption{Treewidth Estimates and Speed-up Ratios for Networks Used in Parallelized Parameter Learning.}
	\label{tab:treewidth_heuristics}
	\begin{tabular}{lcccccc}
		\hline
		Network        & Child & Alarm & Hailfinder & Hepar2 & Win95pts & Pigs \\
		\hline
		Min-degree     & 3     & 4     & 4          & 6      & 8        & 11   \\
		Min-fill-in    & 3     & 4     & 4          & 6      & 8        & 10   \\
		Ratio          & 1.40& 1.81& 2.03     & 1.96 & 2.02   & 2.69 \\
		\hline
	\end{tabular}
\end{table}

\section{Condition Numbers of Covariance Matrices}\label{supp-con}

To assess potential numerical instability in the learned linear Gaussian Bayesian networks, we report the condition numbers of their covariance matrices. For each network, the condition number was computed after learning the network parameters.  

Experiments were conducted on four benchmark networks from the \textit{bnlearn} repository: \textit{ecoli70} (44 variables), \textit{magic-niab} (46 variables), \textit{magic-irri} (64 variables), and \textit{arth150} (107 variables). For each network and sample size, 100 independent datasets were generated from the original network. The parameters were then learned using the proposed decomposition method, and the condition numbers of the resulting covariance matrices were computed. Table~\ref{tab:cov_cond} reports the 5th and 95th percentiles of the condition numbers across these 100 repetitions, presented in the format \textbf{(5\%, 95\%)}.

\begin{table}[htbp]
	\centering
	\caption{Condition numbers (5\% and 95\% percentiles) of the covariance matrices for the learned networks.}
	\label{tab:cov_cond}
	\begin{tabular}{c|cccc}
		\toprule
		Sample Size & ecoli70 & magic-niab & magic-irri & arth150 \\
		\midrule
		500     & (38.90, 3095.70) & (39.88, 3643.56) & (55.55, 5494.25) & (84.56, 3737.75) \\
		1000    & (39.51, 1959.21) & (29.45, 3244.93) & (51.57, 2159.21) & (70.78, 5513.69) \\
		2500    & (38.17, 1760.05) & (34.84, 1067.35) & (55.08, 2558.71) & (117.42, 7435.66) \\
		5000    & (29.76, 980.95)  & (37.99, 836.12)  & (48.23, 1249.18) & (73.31, 4558.45) \\
		7500    & (32.10, 1822.79) & (31.46, 2183.49) & (59.05, 1976.73) & (104.23, 6284.86) \\
		10000   & (37.77, 1533.57) & (37.33, 1443.18) & (44.49, 4280.49) & (75.40, 9266.79) \\
		\bottomrule
	\end{tabular}
\end{table}

As shown in Table~\ref{tab:cov_cond}, the condition numbers remain moderate for most networks and sample sizes, indicating that the learned covariance matrices are generally well-conditioned. For some networks, particularly \textit{arth150}, the upper percentiles at larger sample sizes are relatively high, reflecting variability due to network structure and finite-sample estimation. Nevertheless, even these extreme values remain within a range that permits stable linear Gaussian inference. These results provide quantitative evidence that the proposed parameter learning procedure produces covariance matrices suitable for reliable numerical computation across repeated experiments.

\end{document}